\documentclass{article}

\PassOptionsToPackage{numbers}{natbib} 
\usepackage[sglblindworkshop, final]{neurips_2025}
\usepackage[utf8]{inputenc} % allow utf-8 input
\usepackage[T1]{fontenc}    % use 8-bit T1 fonts
\usepackage{hyperref}       % hyperlinks
\usepackage{url}            % simple URL typesetting
\usepackage{booktabs}       % professional-quality tables
\usepackage{amsfonts}       % blackboard math symbols
\usepackage{nicefrac}       % compact symbols for 1/2, etc.
\usepackage{microtype}      % microtypography
\usepackage{xcolor}         % colors
\usepackage{graphicx}
\usepackage{amsmath} 
\usepackage{bm}
\usepackage{mathtools}
\usepackage{amsthm}
\usepackage{cleveref}
\usepackage{breqn}
\usepackage{caption}
\usepackage[dvipsnames]{xcolor}

\newtheorem{lemma}{Lemma}               % 補題
\newtheorem{proposition}{Proposition}   % 命題
\title{Optimized Learned Count-Min Sketch}
\workshoptitle{Machine Learning for Systems}

\author{%
   Kyosuke Nishishita \quad
  Atsuki Sato \quad
  Yusuke Matsui \\
  The University of Tokyo \\
  \texttt{\{nishishita, a\_sato, matsui\}@hal.t.u-tokyo.ac.jp}
}

\begin{document}

\maketitle

\begin{abstract}
Count-Min Sketch (CMS) is a memory-efficient data structure for estimating the frequency of elements in a multiset. 
Learned Count-Min Sketch (LCMS) enhances CMS with a machine learning model to reduce estimation error under the same memory usage, but suffers from slow construction due to empirical parameter tuning and lacks theoretical guarantees on intolerable error probability. 
We propose Optimized Learned Count-Min Sketch (OptLCMS), which partitions the input domain and assigns each partition to its own CMS instance, with CMS parameters $(\epsilon, \delta)$ analytically derived for fixed thresholds, and thresholds optimized via dynamic programming with approximate feasibility checks. This reduces the need for empirical validation, enabling faster construction while providing theoretical guarantees under these assumptions.
OptLCMS also allows explicit control of the allowable error threshold, improving flexibility in practice. 
Experiments show that OptLCMS builds faster, achieves lower intolerable error probability, and matches the estimation accuracy of LCMS.
\end{abstract}

\section{Introduction}\label{introduction}
Count-Min Sketch (CMS) is a probabilistic data structure that estimates the frequency of each element in a multiset in a memory-efficient way~\cite{cms_applicatin}. 
Frequency estimation is fundamental, and sketch data structures such as CMS are widely used~\cite{estan2003new,schechter2010popularity,goyal2012sketch,talukdar2014scaling,dzogang2015scalable,NetCache,aghazadeh2018mission}.
While CMS is a useful data structure, there is a trade-off between estimation error and memory usage:
reducing memory consumption increases the estimation error, while achieving lower error requires more memory.

Hsu et al.~\cite{hsu2019learning} proposed the Learned Count-Min Sketch (LCMS), which trains a machine learning model to generate approximate frequency scores, using a threshold to decide whether an element is stored in a dictionary or counted by CMS.  
Under a fixed memory budget, LCMS tunes CMS parameters and the threshold via validation data to minimize error, achieving more memory-efficient estimation than CMS. 
However, the construction of LCMS is slow because it requires validation.
Also, LCMS lacks a theoretical guarantee on the probability of an ``intolerable error,'' i.e., that the estimation error exceeds a user-specified threshold. 
Classical CMS provides such bounds, ensuring reliability, but LCMS does not.

We propose the Optimized Learned Count-Min Sketch (OptLCMS) to address these issues. 
Our method targets the Partitioned Learned Count-Min Sketch (PL-CMS)~\cite{nguyenpartitioned}, a recently introduced data structure for the heavy hitter problem, and adapts it for frequency estimation. 
We optimize its parameters by partitioning the score space and formulating an optimization problem to minimize the probability of intolerable error. For fixed thresholds, $\delta$ can be derived analytically via the Karush-Kuhn-Tucker (KKT) conditions, while thresholds themselves are optimized through dynamic programming with approximate feasibility checks. CMS dimensions are obtained by relaxing integer constraints. This approach enables near-analytical construction, reduces build time compared to LCMS while maintaining memory efficiency.
It also offers flexible error control with stronger guarantees.

\section{Related Work}
Learned data structures were introduced by Kraska et al.~\cite{kraska2018case} with ``Learned Indexes,'' which reinterpret classical indexes such as B-Tree~\cite{bayer1970organization,comer1979ubiquitous} as models mapping keys to record positions. Subsequent work extended the idea from one-dimensional~\cite{ALEX,PGM,kipf2020radixspline} to multi-dimensional ones~\cite{nathan2020learning,ding2020tsunami,hidaka}. Kraska et al.~\cite{kraska2018case} also observed that probabilistic data structures, e.g., Bloom Filters~\cite{bloom1970space}, can be learned, inspiring Learned Bloom Filters~\cite{NEURIPS2020_86b94dae,vaidya2020partitioned}. The Partitioned Learned Bloom Filter (PLBF)~\cite{vaidya2020partitioned,sato2023fast} partitions the score space to optimize the false positive rate. Bloom Filters and Count-Min Sketch (CMS) share a space-efficient probabilistic design: the former tests set membership, the latter estimates frequencies. We extend the partitioning and optimization ideas from PLBF to CMS.

In the context of Learned Count-Min Sketch (LCMS), Hsu et al.~\cite{hsu2019learning} proposed combining CMS with a machine learning model, followed by Zhang et al.~\cite{ZHANG2020365}, which uses a threshold to route elements with low scores to CMS and high scores to the model’s estimate. This approach depends heavily on model accuracy and cannot asymptotically eliminate error even with large memory. Nguyen et al.~\cite{nguyenpartitioned} introduced a partitioned LCMS for the heavy hitter problem, differing from frequency estimation. Their threshold is fixed a priori, whereas our method computes it via a dynamic programming approach to approximate the optimum.

\section{Preliminaries}
\paragraph{Multiset.}
A multiset is a set that allows multiple occurrences of the same element. 
For example, $\mathcal{A}=\{a,a,b\}$ and $\mathcal{B}=\{a,b,b\}$ are different, with $a$ occurring twice in $\mathcal{A}$ and once in $\mathcal{B}$. 
All sets in this paper are multisets.

\paragraph{Count-Min Sketch (CMS).}
The parameters of CMS are $\epsilon$ and $\delta$, where $0 < \epsilon$ and $0 < \delta \le 1$. 
CMS estimates the frequency $f(x)$ of an element $x$ in a multiset using a table of width 
$w = \lceil e/\epsilon \rceil$ and depth $d = \lceil \ln(1/\delta) \rceil$ with $d$ hash functions. 
 For each element in a multiset, $d$ hash functions specify $d$ cells.
 The value in each of these cells is incremented by one.
 The estimate $\hat{f}(x)$ is then given by the minimum of the $d$ cell values.
%counts are added to $d$ corresponding cells, and the estimate $\hat{f}(x)$ 
% is the minimum of these counts. 
CMS guarantees $\Pr[\hat{f}(x) - f(x) > \epsilon N] < \delta$ where $N$ is the size of the multiset, 
and requires $w \cdot d \cdot b$ bytes of memory, where $b$ is bytes per cell.
Details of CMS are given in Appendix~\ref{appendix:cms}.

\paragraph{Unique Bucket (UB).}
UB is an exact frequency counter, essentially a dictionary with keys given by the elements and values given by their exact counts, which requires significantly more memory than CMS.

\section{Proposed Method}

\begin{figure*}[t]
    \centering
    \includegraphics[width=\textwidth]{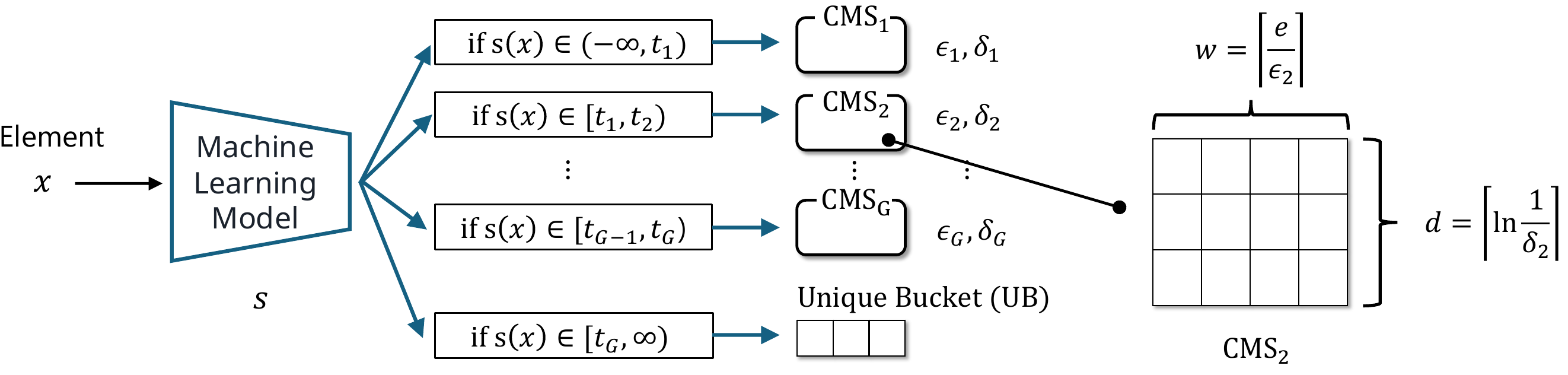}
    \caption{Partitioned Learned Count-Min Sketch}
    \label{ours}
\end{figure*}
\label{method}
We propose the Optimized Learned Count-Min Sketch (OptLCMS), which extends the Partitioned Learned Count-Min Sketch (PL-CMS)~\cite{nguyenpartitioned}. 
While PL-CMS was developed for the heavy hitter problem, we adapt it to frequency estimation and formulate a optimization problem to determine its parameters. 
By partitioning the score space and analytically solving this problem, we minimize the probability that the estimation error exceeds a user-specified threshold under a fixed memory budget.

\paragraph{Structure.}
PL-CMS consists of a pre-trained machine learning model, $G$ CMS tables, and a unique bucket (UB) for high-frequency elements (\Cref{ours}). 
Elements are routed based on the model score: high-score items go to UB for exact counting, others to the CMS of the corresponding score group. 
This design preserves the CMS property $\hat{f}(x) \ge f(x)$.
We denote by $\bm{t} = [t_1,\dots,t_G]^\top$ the vector of partition thresholds with $t_1 < t_2 < \dots < t_G$, and let $\bm{\epsilon} = [\epsilon_1,\dots,\epsilon_G]^\top$ and $\bm{\delta} = [\delta_1,\dots,\delta_G]^\top$ be the CMS parameter vectors for the $G$ groups.  

\paragraph{Optimization formulation.}
Our task is to efficiently determine optimal thresholds $\bm{t}$ and CMS parameters $\bm{\epsilon}, \bm{\delta}$, given the pre-trained ML model and the input multiset, under a memory budget.  
To this end, we formulate the following optimization problem

The upper bound on the intolerable error probability is $\sum_{g=1}^{G} \delta_g q_g$
which we minimize subject to the memory budget and CMS constraints:
\begin{align}
    \min_{\bm{t},\bm{\epsilon},\bm{\delta}} & \quad  \sum_{g=1}^{G} \delta_g q_g \label{eq:object}\\
    \text{subject to }
    & \quad b\sum_{g=1}^{G} \left\lceil \frac{e}{\epsilon_g}\right\rceil \left\lceil\ln{\frac{1}{\delta_g}}\right\rceil  + c n = M \label{constrain:memory}\\ 
    & \quad \delta_g \leq 1, \quad g = 1, \dots, G\label{constrain:delta} \\
    & \quad \epsilon_g = \frac{\epsilon N}{N_g}, \quad g = 1, \dots, G \label{constrain:epsilon}
\end{align}
Here, $n$ denotes the number of unique elements classified into the UB, and $c$ is the memory usage per element. 
Thus, $cn$ corresponds to the total memory consumption (in bytes) of the UB. 
\Cref{constrain:memory} specifies the memory budget, where $M$ is the total available memory. 
\Cref{constrain:delta} enforces the CMS condition $\delta_g \leq 1$, 
and \Cref{constrain:epsilon} ensures that each CMS complies with the global allowable error.
The ceiling functions enforce integer table dimensions in practice; however, for theoretical analysis, we relax them to continuous variables, introducing only a negligible non-convexity in implementation.
The derivation of this formulation is in Appendix~\ref{appendix:optimization}.

\paragraph{Analytical solution for $\bm{\delta}$.}
With fixed thresholds $\bm{t}$, the problem is convex in $\bm{\delta}$ and can be solved analytically via the Karush-Kuhn-Tucker (KKT) conditions.  
The detailed derivation is provided in Appendix~\ref{appendix:optimization}.  
The closed-form solution for each $g \in \{1, \dots, G\}$ is:
\begin{equation}\label{eq:delta}
\delta_g =
\min\left\{1,\;
\frac{1}{q_g\epsilon_g}
\exp\left[
-\frac{\frac{M-cn}{be} - I}
{\sum_{g\in\mathcal{D}} \frac{1}{\epsilon_g}}
\right]
\right\},
\end{equation}
where
\begin{equation}
    I = \sum_{g\in\mathcal{D}} \frac{1}{\epsilon_g} \ln\left(q_g\epsilon_g\right), \quad
    \mathcal{D} = \left\{ g \mid \delta_g < 1 \right\}.
\end{equation}
This gives $\bm{\epsilon}$ and $\bm{\delta}$ for the $G$ CMSs when $\bm{t}$ is fixed.

\paragraph{Threshold optimization.}
Substituting $(\bm{\epsilon},\bm{\delta})$ into the objective yields an expression involving the Kullback–Leibler (KL) divergence between the data and query distributions of each group.  
The optimal thresholds are found by maximizing this divergence via dynamic programming (DP), ensuring $\delta_g < 1$ for all groups.  
The full derivation, DP recurrence, and proof of correctness are in Appendix~\ref{appendix:optimization}.

\section{Experiments}\label{sec:experiments}
We evaluate OptLCMS on the AOL query log (21M queries, 3.8M unique terms, Zipfian distribution) under two query patterns: \textit{uniform} (each term once) and \textit{frequency-weighted} (proportional to term frequency).  
All methods use the learned model of~\cite{hsu2019learning}, with amortized model size $0.0152$ MB.
% While this model size is excluded from the theoretical memory budget in \Cref{constrain:memory} since it is constant across methods.
% In our experiments, however, we report the total memory usage including the model size on the horizontal axis (\Cref{fig:memory-prob} and \ref{fig:memory-error}) for consistency.
This constant cost is excluded from \Cref{constrain:memory} but included in the reported memory usage for \Cref{fig:memory-prob} and \ref{fig:memory-error}.
Implementation and dataset details are provided in Appendix~\ref{appendix:experimental setup}.

\paragraph{Baselines.}
We compare against: (i) CMS~\cite{cms_applicatin}, tuned to minimize error for a fixed memory budget; (ii) LCMS~\cite{hsu2019learning}, using the same model; and (iii) OptLCMS, with $G=10$.
For each $M$, $\epsilon$ is set to the smallest value feasible under CMS, namely $\epsilon = \frac{e}{M}$.

\begin{figure*}[t]
    \centering
    \includegraphics[width=0.86\textwidth]{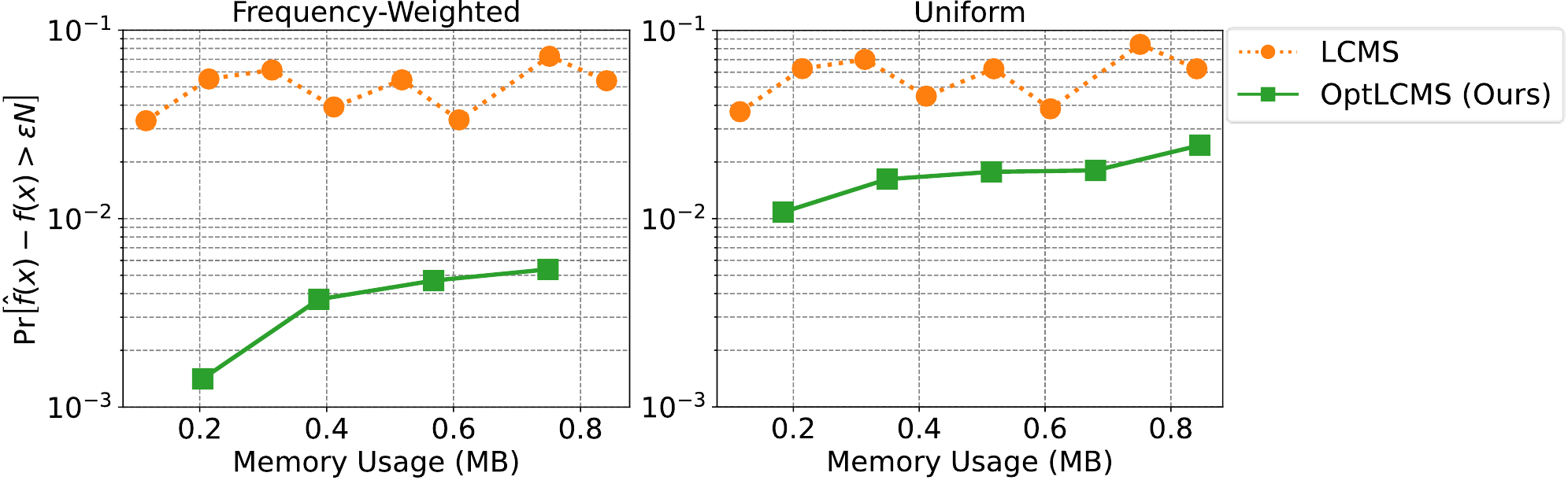}
    \caption{Memory Usage vs. intolerable error probability, defined as
    $\mathrm{Pr}\left[\hat{f}(x)-f(x)>\epsilon N\right]$. A lower position on the vertical axis indicates better performance.}
    \label{fig:memory-prob}
    
    \vspace{1em}
    
    \includegraphics[width=0.86\textwidth]{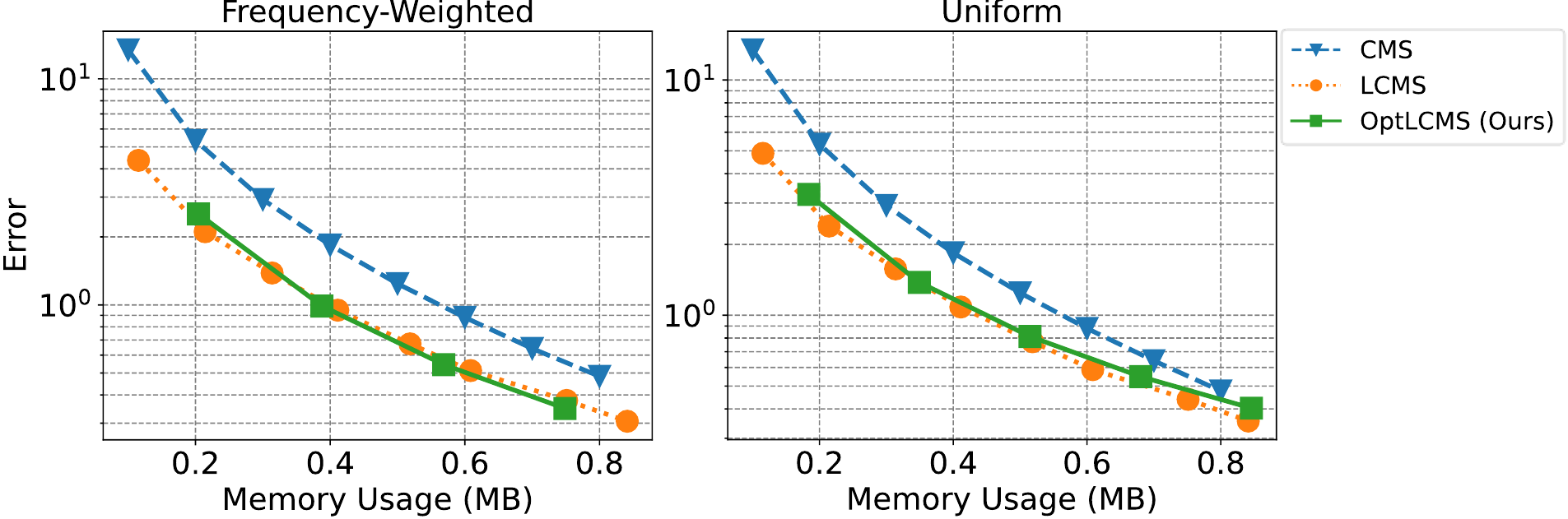}
    \caption{Memory Usage vs. Error. Curves closer to the bottom-left corner perform better.}
    \label{fig:memory-error}

    \vspace{1em}

    \begin{tabular}{@{}lll@{}}
    \toprule
        Data Structure & Construction [s] & distribution\\ \midrule
        LCMS & 10.712 & frequency-weighted\\
        LCMS & 10.250 & uniform \\
        OptLCMS (Ours) & 0.003 \textcolor{ForestGreen}{\textbf{(-10.709)}} &frequency-weighted\\
        OptLCMS (Ours) & 4.873  \textcolor{ForestGreen}{\textbf{(-5.377)}} &uniform \\ \bottomrule
    \end{tabular}
    \captionof{table}{Comparison of construction time}
    \label{tab:time}
\end{figure*}

\paragraph{Results.}
\Cref{fig:memory-prob} shows that our OptLCMS consistently yields the lowest intolerable error probability while matching LCMS in average error, under both query patterns.
OptLCMS achieves up to about 20× smaller intolerable error probability compared to LCMS.
\Cref{fig:memory-error} further confirms that minimizing the upper bound on intolerable error does not sacrifice average error performance.

\Cref{tab:time} reports construction times.  
Unlike LCMS, which requires exhaustive threshold and parameter search with actual measurements, OptLCMS computes optimal parameters analytically from the validation data distribution, yielding orders-of-magnitude faster construction for frequency-weighted queries and more than twice as fast for uniform queries.

\section{Conclusion and Future Work}\label{conclusion and future work}
We presented OptLCMS, a theoretically grounded variant of LCMS~\cite{hsu2019learning} that analytically determines CMS parameters by partitioning the score space, thereby minimizing the upper bound on the probability of intolerable error under a fixed memory budget.  
This approach removes the need for empirical validation, enabling much faster construction while maintaining comparable estimation accuracy to LCMS and providing stronger guarantees and flexibility.  
Due to time and scope constraints, we did not evaluate on the CAIDA dataset~\cite{hsu2019learning} or compare with confidence-aware LCMS~\cite{aamand2023improved}; extending our evaluation to these settings is left for future work.

\newpage

\bibliographystyle{unsrtnat}  
\bibliography{reference}

\newpage

%%%%%%%%%%%%%%%%%%%%%%%%%%%%%%%%%%%%%%%%%%%%%%%%%%%%%%%%%%%%

\appendix\label{appendix}

\section{Count-Min Sketch (CMS) Details}\label{appendix:cms}
Count-Min Sketch (CMS) is a probabilistic data structure for estimating the frequency $f(x)$ of an element $x$ in a multiset $\mathcal{U}$. 
CMS maintains a two-dimensional array $C$ of width
\begin{equation}
w \coloneqq \left\lceil \frac{e}{\epsilon} \right\rceil
\end{equation}
and depth
\begin{equation}
d \coloneqq \left\lceil \ln{\frac{1}{\delta}} \right\rceil,
\end{equation}
where $\epsilon > 0$ and $0 < \delta \le 1$ are user-specified parameters controlling accuracy and confidence, respectively.  

\subsection{Hash Functions}
We prepare $d$ pairwise-independent hash functions
\begin{equation}
h_i: \mathcal{S} \to \{0, 1, \dots, w-1\}, \quad i = 0, 1, \dots, d-1,
\end{equation}
where $\mathcal{S}$ denotes the set of all unique elements that can appear in $\mathcal{U}$.

\subsection{Update Procedure}
For each occurrence of an element $x$, the count in each row is incremented as:
\begin{equation}
C[i,\, h_i(x)] \ \leftarrow \ C[i,\, h_i(x)] + 1, \quad \forall i \in \{0, \dots, d-1\}.
\end{equation}

\subsection{Query Procedure}
The estimated frequency $\hat{f}(x)$ is given by:
\begin{equation}
\hat{f}(x) \coloneqq \min_{i \in \{0, \dots, d-1\}} C\big[i,\, h_i(x)\big],
\end{equation}
which mitigates the effect of overcounting caused by hash collisions.

\subsection{Error Guarantee}
By construction, CMS guarantees $\hat{f}(x) \ge f(x)$ and satisfies the probabilistic bound:
\begin{equation}\label{eq:cms-error-appendix}
\Pr\left[ \hat{f}(x) - f(x) > \epsilon N \right] < \delta,
\end{equation}
where $N \coloneqq |\mathcal{U}|$ is the total number of elements in the multiset.

\subsection{Memory Usage}
The total memory consumption is:
\begin{equation}
\mathrm{Memory} = \left\lceil \frac{e}{\epsilon} \right\rceil \cdot \left\lceil \ln{\frac{1}{\delta}} \right\rceil \cdot b \quad [\mathrm{bytes}],
\end{equation}
where $b$ is the memory usage per counter (in bytes).

\section{Derivation of the Optimization Problem}\label{appendix:optimization}
\subsection{Objective: Probability Upper Bound to Minimize}
Our goal is to minimize the probability that the estimation error of the LCMS exceeds the allowable error $\epsilon N$.  
We refer to this quantity as the \textit{intolerable error probability}.
To analyze this, we partition the score space from the learned model into $G+1$ groups, where $G$ is a hyperparameter, and associate a Count-Min Sketch (CMS) with each group.

Let the threshold vector be $\bm{t} = [t_1,\dots,t_G]^{\top}$, where $t_1 < t_2 < \dots < t_G$.  
For the $g$-th group, we define the CMS parameters $\epsilon_g$ and $\delta_g$, and denote the sets handled by each CMS as $\mathcal{U}_g$ with size $N_g = |\mathcal{U}_g|$.  
We also define the query distribution $q_g = \mathrm{Pr}[x \in \mathcal{U}_g]$.  
The parameter condition
\begin{equation}
    \epsilon_g N_g = \epsilon N
\end{equation}
is imposed so that the allowable error for each CMS matches that of the entire system.  

We consider the probability that the estimation error exceeds $\epsilon N$:
\begin{align}
\mathrm{Pr}\left[\hat{f}(x) - f(x) > \epsilon N \right]
&= \sum_{g=1}^{G} \mathrm{Pr}\left[\hat{f}(x) - f(x) > \epsilon N \mid x \in \mathcal{U}_g \right] \mathrm{Pr}[x \in \mathcal{U}_g] \\
&< \sum_{g=1}^{G} \delta_g q_g ,
\end{align}
where the last inequality follows from applying the CMS error bound to each group.  
Thus, minimizing the intolerable error probability reduces to minimizing $\sum_{g=1}^G \delta_g q_g$ over the partitioning thresholds $\bm{t}$ and CMS parameters $(\epsilon_g, \delta_g)$, subject to the overall memory constraint.

\subsection{Problem Formulation}
The optimization problem is:
\begin{align}
    \min_{\bm{t},\bm{\epsilon},\bm{\delta}} & \quad  \sum_{g=1}^{G} \delta_g q_g \label{eq:object}\\
    \text{s.t. }
    & \quad b\sum_{g=1}^{G} \left\lceil \frac{e}{\epsilon_g}\right\rceil \left\lceil\ln{\frac{1}{\delta_g}}\right\rceil  + c n = M \\ 
    & \quad \delta_g \leq 1, \quad g = 1, \dots, G \\
    & \quad \epsilon_g = \frac{\epsilon N}{N_g}, \quad g = 1, \dots, G 
\end{align}
where $n$ is the number of unique elements in the UB, $c$ the per-element memory, and $M$ the total budget.

\subsection{Solution of the optimization problem}
\paragraph{Lagrangian Function}
We now apply the Karush–Kuhn–Tucker (KKT) conditions to derive the optimal solution.
This involves setting the derivatives of the Lagrangian with respect to each variable to zero and applying the complementary slackness condition.
The Lagrangian function for the problem is given by:
\begin{equation}
\mathcal{L} (\bm{\delta},\bm{\mu},\lambda) 
= \sum_{g=1}^{G} \delta_g q_g 
    + \sum_{g=1}^{G}{\mu_g\left(\delta_g - 1\right)}
      + \lambda \left(  b\sum_{g=1}^{G} \frac{e}{\epsilon_g} \ln{\frac{1}{\delta_g}} + c n - M \right).
\end{equation}

\paragraph{First-Order Conditions}
Taking the partial derivatives of the Lagrangian with respect to $\delta_g$, $\mu_g$, and $\lambda$, we obtain the following conditions.

The derivative with respect to $\delta_g$ is:
\begin{equation}
    \frac{\partial \mathcal{L}}{\partial \delta_g} = q_g + \mu_g - \lambda b \frac{e}{\epsilon_g} \frac{1}{\delta_g}.
\end{equation}
Setting this equal to zero for the optimal condition:
\begin{equation}\label{eq:dl-dd}
    q_g + \mu_g - \lambda b \frac{e}{\epsilon_g} \frac{1}{\delta_g} = 0.
\end{equation}

The derivative with respect to $\lambda$ is:
\begin{equation}
    \frac{\partial \mathcal{L}}{\partial \lambda} = b \sum_{g=1}^{G} \frac{e}{\epsilon_g} \ln \frac{1}{\delta_g} + c n - M.
\end{equation}
Setting this equal to zero for the optimal condition:
\begin{equation}\label{appendix:memory-const}
    b \sum_{g=1}^{G} \frac{e}{\epsilon_g} \ln \frac{1}{\delta_g} + c n - M = 0.
\end{equation}

The derivative with respect to $\mu_g$ is:
\begin{equation}
    \frac{\partial \mathcal{L}}{\partial \mu_g} = \delta_g - 1.
\end{equation}

\paragraph{KKT Conditions}
From the complementary slackness condition, we obtain the following:
\begin{equation}
    \mu_g (\delta_g - 1) = 0.
\end{equation}
This implies that either $\mu_g = 0$ or $\delta_g = 1$, meaning that the constraint is either inactive or active.
From the non-negativity of the Lagrange multiplier, we have $\mu_g \geq 0$.
From the non-negativity of the constraint, it follows that $\delta_g \leq 1$.

Consider the case $\mu_g = 0$.
When $\mu_g = 0$, from the complementary slackness condition, we have $\delta_g < 1$. 
The optimization condition \eqref{eq:dl-dd} becomes:
\begin{equation}
    q_g - \lambda b \frac{e}{\epsilon_g} \frac{1}{\delta_g} = 0,
\end{equation}
which can be solved for $\delta_g$ as:
\begin{equation}
    \delta_g = \frac{\lambda b e}{q_g \epsilon_g}.
\end{equation}
We then check whether this $\delta_g$ satisfies the condition $\delta_g \leq 1$. If $\delta_g \geq 1$, this case is not valid.

Consider the case $\delta_g = 1$.
When $\delta_g = 1$, the complementary slackness condition implies that $\mu_g \geq 0$. The optimization condition becomes:
\begin{equation}
    q_g + \mu_g - \lambda b \frac{e}{\epsilon_g} = 0,
\end{equation}
which gives:
\begin{equation}
    \mu_g = \lambda b \frac{e}{\epsilon_g} - q_g.
\end{equation}

\paragraph{Derivation of $\delta_g$}
Let $\mathcal{D}$ be the set of $g \in \left\{1,\dots,G\right\}$ such that $\delta_g < 1$, i.e.,
\begin{equation}
    \mathcal{D} = \left\{g \mid \delta_g<1\right\}.
\end{equation}
In this case, for $g \in \mathcal{D}$, $\delta_g = \frac{\lambda b e}{q_g \epsilon_g}$, and for $g \notin \mathcal{D}$, $\delta_g = 1$.

Substituting $\delta_g$ into \Cref{appendix:memory-const}, we get:
\begin{align}
    b \sum_{g\in\mathcal{D}} \frac{e}{\epsilon_g} \ln {\frac{1}{\frac{\lambda b e}{q_g \epsilon_g}}} + c n - M &= 0 \\
    b \sum_{g\in\mathcal{D}} \frac{e}{\epsilon_g} \ln {\frac{q_g\epsilon_g}{\lambda b e}} + c n - M &= 0 \\
    \sum_{g\in\mathcal{D}}{\frac{1}{\epsilon_g}\ln{\left(q_g\epsilon_g\right)}} - \sum_{g\in\mathcal{D}}\frac{1}{\epsilon_g}\ln{\left(\lambda b e\right)}&= \frac{M-cn}{be} \\
    -\sum_{g\in\mathcal{D}}\frac{1}{\epsilon_g}\ln{\left(\lambda b e\right)} &= \frac{M-cn}{be} - \sum_{g\in\mathcal{D}}{\frac{1}{\epsilon_g}\ln{\left(q_g\epsilon_g\right)}} \\
    \lambda b e &= \exp{\left[-\frac{\frac{M-cn}{be} - \sum_{g\in\mathcal{D}}{\frac{1}{\epsilon_g}\ln{\left(q_g\epsilon_g\right)}}}{\sum_{g\in\mathcal{D}}{\frac{1}{\epsilon_g}}}\right]}
\end{align}
Therefore, for $g \in \mathcal{D}$, $\delta_g$ is given by the following:
\begin{equation}
    \delta_g = \frac{1}{q_g\epsilon_g}\exp{\left[-\frac{\frac{M-cn}{be} -\sum_{g\in\mathcal{D}}{\frac{1}{\epsilon_g}\ln{\left(q_g\epsilon_g\right)}}}{\sum_{g\in\mathcal{D}}{\frac{1}{\epsilon_g}}}\right]}.
\end{equation}

From this, $\delta_g$ is expressed as follows:
\begin{equation}
\delta_g = \\
\min\left\{1,\quad
\frac{1}{q_g\epsilon_g}\exp{\left[-\frac{\frac{M-cn}{be} -\sum_{g\in\mathcal{D}}{\frac{1}{\epsilon_g}\ln{\left(q_g\epsilon_g\right)}}}{\sum_{g\in\mathcal{D}}{\frac{1}{\epsilon_g}}}\right]}
\right\}.
\end{equation}

\subsection{Threshold Determination Method}
Given a threshold vector $\bm{t}=\left[t_1,\dots,t_G \right]^{\top}$, we can analytically derive the CMS parameters as the solution to the optimization problem in \Cref{eq:object}, as shown in \Cref{eq:delta}.
Thus, the next task is to find the optimal threshold vector $\bm{t}=\left[t_1,\dots,t_G \right]^{\top}$. 

To determine the threshold, we substitute $\bm{\epsilon}$ and $\bm{\delta}$ into the objective function, which we aim to minimize. 
We rewrite the objective function \eqref{eq:object} as follows:
\begin{equation}
    \sum_{g \in \mathcal{\overline{D}} }{q_g} 
    + \sum_{g \in \mathcal{D}}{\frac{1}{\epsilon_g}
    \exp{\left[-\frac{\epsilon N (M-cn)}{be\sum_{g \in \mathcal{D}}N_g}\right]}
    \exp{\left[\frac{I}{\sum_{g \in \mathcal{D}}{\frac{1}{\epsilon_g}}}\right]}},
\end{equation}
where $\overline{D} = \left\{g\mid\delta_g = 1\right\}$.

Considering that the task to solve is element frequency estimation, we impose the condition $\delta_g \neq 1$.
This is because $\delta_g = 1$ implies that the CMS table for group $g$ does not exist. 
While $\delta_g = 1$ may be a valid result for minimizing the probability that the estimation error exceeds the allowable error, a table is necessary to count the element frequencies. 
Therefore, we consider $\delta_g < 1$ going forward.

In this case, the objective function to be minimized is as follows:
\begin{equation}\label{func:object}
    \sum_{g =1}^{G}{\frac{1}{\epsilon_g}
    \exp{\left[-\frac{\epsilon N (M-cn)}{be\sum_{g =1}^{G}N_g}\right]}
    \exp{\left[\frac{I}{\sum_{g = 1}^{G}{\frac{1}{\epsilon_g}}}\right]}}.
\end{equation}
Let the size of the multiset consisting of the elements classified into the UB be denoted as $N_{\mathrm{UB}}$. 
Note that $n$ refers to the size of the set of elements classified into the UB, while $N_{\mathrm{UB}}$ represents the size of the multiset, e.g. $n=2$ and $N_{\mathrm{UB}} = 3$ for $\left\{x_1, x_2,x_2 \right\}$.
Additionaly, for each $g \in \left\{1,\dots,G\right\}$ we define new variables as follows:
\begin{align}
    u_g = \frac{N_g}{N-N_{\mathrm{UB}}}, \\
    v_g = \frac{q_g}{1 - q_{\mathrm{UB}}}.
\end{align}
$u_g$ represents the size of the multiset handled by the CMS for group $g$ relative to the size of the multiset handled by the CMS, and it satisfies $\sum_{g=1}^{G}{v_g} = 1$.
Additionally, $v_g$ represents the proportion of queries handled by the CMS for group $g$ out of all queries processed by the CMS, and it satisfies $\sum_{g=1}^{G}{v_g} = 1$. Substituting these into $\epsilon_g$ and $q_g$, we obtain the following:
\begin{align}
    \epsilon_g = \frac{\epsilon N}{(N-N_{\mathrm{UB}})u_g}, \\
    q_g = v_g (1-q_{\mathrm{UB}}).
\end{align}

Substituting these into \Cref{func:object} and simplifying, the objective function to minimize becomes:
\begin{align}\label{func:minimum}
    \left(1-q_{\mathrm{UB}}\right)\exp{\left[-\frac{\epsilon N(M-cn)}{be(N-N_{\mathrm{UB}})}\right]}
    \exp{\left[-\sum_{g=1}^{G}u_g\ln{\frac{u_g}{v_g}}\right]}.
\end{align}
By determining the threshold $t_G$ that defines the boundary between the UB and the CMS, we can automatically determine the variables $q_{\mathrm{UB}},n$, and $N_{\mathrm{UB}}$.
Therefore, what \Cref{func:minimum} suggests is that by setting $t_G$,  the partitioning of the score space that is smaller than $t_G$ should be done in such a way that $\sum_{g=1}^{G}u_g\ln{\frac{u_g}{v_g}}$ is maximized.
The term $\sum_{g=1}^{G}u_g\ln{\frac{u_g}{v_g}}$ represents the Kullback-Leibler (KL) divergence.

\subsection{Derivation of the Threshold}
This maximization problem of the KL divergence can be efficiently solved using dynamic programming (DP) as follows.
\begin{dmath}
\label{DP}
    DP (s,p) = \max_{y} \left\{DP(y,p-1) + u(y,s) \ln{\frac{u(y,s)}{v(y,s)}}\right\}.
\end{dmath}
Here, $DP(s,p)$ represents the maximum divergence when the score space with scores less than $s$ is divided into no more than $p$ parts.
Additionally, $u(y,s)$ represents the proportion of the size of the multiset consisting of elements with scores in the interval $[y,s)$, relative to the size of the multiset processed by the CMS.
Meanwhile, $v(y,s)$ represents the proportion of queries processed by the CMS that correspond to elements with scores in the interval $[y,s)$.

The following explains the $DP(s,p)$ process. First, prepare a list to store the sorted threshold candidates in advance.
\begin{enumerate} 
    \item Assign $DP(s,p-1)$ to $DP(s,p)$. 
    \item Iterate through the threshold candidates that are less than $s$. Let this be denoted as $y$. 
    \item Consider dividing the score space using the threshold $y$, and compute $DP(y,p-1) + u(y,s) \ln{\frac{u(y,s)}{v(y,s)}}$. 
    \item If $DP(y,p) < DP(y,p-1) + u(y,s) \ln{\frac{u(y,s)}{v(y,s)}}$, verify whether the CMS handling the score space $[y,s)$ (which is denoted as group $g$) satisfies the condition for $\delta_g$ \label{algo:check}. 
    \item If the condition is satisfied, update $DP(y,p)$ as $DP(y,p) = DP(y,p-1) + u(y,s) \ln{\frac{u(y,s)}{v(y,s)}}$. \end{enumerate}
This DP solution would return an analytical optimal solution if the condition $\delta_g < 1$ were not present, as the process in step \Cref{algo:check} would not be required. 
However, since CMS requires the condition $\delta_g < 1$, the process in step \Cref{algo:check} is necessary.

Next, we describe how to check whether each $\delta_g$ satisfies the condition $\delta_g < 1$ when adopting the new threshold $y$.
Now, we are considering an efficient solution using DP, so the division of the score space $[s,t_G)$ is unknown. Therefore, at this stage, it is impossible to determine $\delta_g$ precisely.
Therefore, we approximate $\delta_g$ as $\hat{\delta}_g$ as follows:
\begin{dmath}
    \hat{\delta}_g = \frac{u_g}{v_g}\exp{\left[-\frac{\epsilon N(M-cn)}{be(N-N_{\mathrm{UB}})}\right]}\exp\left[-\left(DP(s,p)+ u(s,t_G)\ln{\frac{u(s,t_G)}{v(s,t_G)}}\right)\right].
\end{dmath}
If $\hat{\delta}_g < 1 \implies \delta_g < 1$ holds, then regardless of how the unscanned score space is divided, we only need to check that $\hat{\delta}_g < 1$.
To demonstrate that $\hat{\delta}_g < 1 \implies \delta_g < 1$, we derive the following lemma.
\begin{lemma}
\label{lem:loginequality}
For any positive numbers $j,j_1,j_2,k,k_1,k_2 > 0$, the following inequality holds:
\[
j \ln\left(\frac{j}{k}\right) \leq j_1 \ln\left(\frac{j_1}{k_1}\right) + j_2 \ln\left(\frac{j_2}{k_2}\right),
\]
where $j = j_1+j_2$ and $k = k_1 + k_2$
\end{lemma}
\begin{proof}
By using Jensen's inequality, since the logarithmic function $\ln(x)$ is concave, the following inequality holds:
\begin{equation}
    \ln\left(w_1 x_1 + w_2 x_2\right) \ge w_1 \ln(x_1) + w_2 \ln(x_2).
\end{equation}
Let $w_1 = \frac{j_1}{j_1 + j_2}$, $w_2 = \frac{j_2}{j_1 + j_2}$, $x_1 = \frac{k_1}{j_1}$, and $x_2 = \frac{k_2}{j_2}$, then
\begin{align*}
    \ln\left(\frac{k_1 + k_2}{j_1 + j_2}\right) \geq \frac{j_1}{j_1 + j_2} \ln\left(\frac{k_1}{j_1}\right) + \frac{j_2}{j_1 + j_2} \ln\left(\frac{k_2}{j_2}\right).
\end{align*}

By multiplying both sides by $j_1+j_2$ and applying the logarithmic property of division, we obtain
\begin{align*}
    (j_1 + j_2) \ln\left(\frac{j_1 + j_2}{k_1 + k_2}\right) \le j_1 \ln\left(\frac{j_1}{k_1}\right) + j_2 \ln\left(\frac{j_2}{k_2}\right).
\end{align*}
This is the inequality to be proven, which is equal to
\begin{align*}
    j\ln\left(\frac{j}{k}\right) \leq j_1 \ln\left(\frac{j_1}{k_1}\right) + j_2 \ln\left(\frac{j_2}{k_2}\right).
\end{align*}
\end{proof}

\begin{proposition}
\label{prop:delta-hat}
    \begin{align*}
        \hat{\delta}_g < 1 \implies \delta_g < 1.
    \end{align*}
\end{proposition}
\begin{proof}
    From Lemma~\ref{lem:loginequality}, the following inequality holds:
     \begin{align*}
        u(s,t_G)\ln{\frac{u(s,t_G)}{v(s,t_G)}} \le \sum_{g=p}^{G} {u_g\ln{\frac{u_g}{v_g}}}.
    \end{align*}
    Therefore, when the interval $[s,t_G)$ is divided, the KL divergence increases compared to the case when no division is made. Thus, we have:
     \begin{align*}
        \hat{\delta}_g < 1 \implies \delta_g < 1 .
    \end{align*}
\end{proof}
What this proposition shows is that if $\hat{\delta_g}<1$ holds at the current stage,then $\delta_g<1$ will hold regardless of how the remaining unexamined score space is divided. 
Therefore, the new threshold can be determined as the value of $y$ that maximizes $DP(s,p)$, when the interval $[y,s)$ is treated as a new group, and the corresponding $\hat{\delta}_g$ for this group satisfies $\hat{\delta}_g<1$.
This allows for the efficient approximation of the optimal threshold.

\section{Experimental Setup}\label{appendix:experimental setup}

\subsection{Dataset}
We use the \textbf{AOL query log dataset}, which contains 21 million search queries collected from approximately 650{,}000 users over 90 days.
\begin{itemize}
    \item \textbf{Unique terms:} 3.8 million
    \item \textbf{Examples:} ``google'', ``yahoo'', ``amazon.com''
    \item \textbf{Distribution:} Follows \textit{Zipf's law}, which is also observed in web traffic patterns~\cite{adamic2002zipf} and natural language word frequencies~\cite{piantadosi2014zipf}.
    \item \textbf{Data split:} Queries from the 5th day are used for training/parameter tuning, and queries from the 50th day are used for evaluation/counting.
\end{itemize}

\subsection{Query Distributions}
We evaluate two types of query distributions:
\begin{enumerate}
    \item \textbf{Uniform Distribution:} Each unique term is queried exactly once.
    \item \textbf{Frequency-Weighted Distribution:} Each unique term is queried according to its observed frequency in the dataset.
\end{enumerate}

\subsection{Implementation Details}
\begin{itemize}
    \item \textbf{Unique Bucket (UB):} Implemented as an open-addressing hash table without bucket expansion.
    \item \textbf{Per-element UB memory cost:} $c = 20\ \mathrm{bytes}$.
    \item \textbf{Learned model:} We adopt the model from Hsu et al.~\cite{hsu2019learning}, trained on the first 5 days of AOL query logs.
    \item \textbf{Model amortized memory usage:} $0.0152\ \mathrm{MB}$ over 90 days (following Hsu et al.~\cite{hsu2019learning}).
    \item \textbf{Hardware:} Intel(R) Core(TM) Ultra 7 155H CPU, 16 cores, 3.80 GHz.
\end{itemize}

\subsection{Comparison Methods}
We compare the following methods under the same total memory budget:
\begin{itemize}
    \item \textbf{Count-Min Sketch (CMS)}~\cite{cms_applicatin}:
    \begin{itemize}
        \item Hyperparameters: $\epsilon, \delta$.
        \item Under a fixed memory budget, parameters that minimize estimation error are non-trivial to derive.
        \item We empirically search multiple parameter settings and report the best-performing configuration.
    \end{itemize}
    \item \textbf{Learned Count-Min Sketch (LCMS)}~\cite{hsu2019learning}:
    \begin{itemize}
        \item Hyperparameter: memory usage.
    \end{itemize}
    \item \textbf{Optimized LCMS (OptLCMS)} (proposed):
    \begin{itemize}
        \item Hyperparameters: memory usage, allowable error $\epsilon$, and maximum number of partitions $G$.
        \item $\epsilon$ is set to the smallest value achievable by CMS under the same memory ($\epsilon = e / M$).
        \item $G = 10$ in all experiments.
    \end{itemize}
\end{itemize}

\end{document}